\RequirePackage{fix-cm}
\documentclass[twocolumn]{svjour3}          

\smartqed  
\usepackage{graphicx}
\journalname{SN Computer Science}

\usepackage{amsmath,amsfonts,amssymb}
\usepackage{nicefrac}
\usepackage{cite}
\usepackage{hyperref}

\def\Dist{\mathop{\rm Dist}}
\def\A{\mathop{\rm A}}
\def\sinc{\mathop{\rm sinc}}

\begin{document}
\title{Theoretical Bounds on Data Requirements for the Ray-Based Classification}

\author{Brian J Weber \and 
        Sandesh S Kalantre  \and 
        Thomas McJunkin \and 
        Jacob M. Taylor \and 
        Justyna P. Zwolak 
}

\authorrunning{B. J. Weber,  S. S. Kalantre,  T. McJunkin,  J. M. Taylor, J. P. Zwolak} 
    
\institute{B. J. Weber \at
            Institute of Mathematical Sciences, ShanghaiTech University,
            Shanghai, 201210, China \\
            \email{bjweber@shanghaitech.edu.cn}
            \and
            S. S. Kalantre \at
            Joint Quantum Institute, University of Maryland,\\
            College Park, MD 20742, USA
            \and
            T. McJunkin \at
            Department of Physics, University of Wisconsin, \\
            Madison, WI 53706, USA
            \and
            J. M. Taylor \at
            National Institute of Standards and Technology,\\
            Gaithersburg, MD 20899, USA
            \and 
            J. P. Zwolak \at
            National Institute of Standards and Technology,\\
            Gaithersburg, MD 20899, USA \\
            \email{jpzwolak@nist.gov}
}
\maketitle

\begin{abstract}
The problem of classifying high-dimensional shapes in real-world data grows in complexity as the dimension of the space increases. For the case of identifying convex shapes of different geometries, a new classification framework has recently been proposed in which the intersections of a set of one-dimensional representations, called \emph{rays}, with the boundaries of the shape are used to identify the specific geometry. 
This ray-based classification (RBC) has been empirically verified using a synthetic dataset of two- and three-dimensional shapes (Zwolak {\it et al.} in Proceedings of Third Workshop on Machine Learning and the Physical Sciences (NeurIPS 2020), Vancouver, Canada {[December 11, 2020]}, arXiv:2010.00500, 2020) and, more recently, has also been validated experimentally (Zwolak {\it et al.}, PRX Quantum {\bf 2}:020335, 2021).
Here, we establish a bound on the number of rays necessary for shape classification, defined by key angular metrics, for arbitrary convex shapes. 
For two dimensions, we derive a lower bound on the number of rays in terms of the shape’s length, diameter, and exterior angles. 
For convex polytopes in $\mathbb{R}^N$, we generalize this result to a similar bound given as a function of the dihedral angle and the geometrical parameters of polygonal faces. 
This result enables a different approach for estimating high-dimensional shapes using substantially fewer data elements than volumetric or surface-based approaches. 
\end{abstract}


\section*{Introduction}\label{sec:intro}
The problem of recognizing objects within images has received immense and growing attention in the literature. 
Aside from visual object recognition in two and three dimensions in real-world applications, such as in medical images segmentation or in self-driving cars, recognizing and classifying objects in $N$ dimensions can be important in scientific applications. 
A problem arises in cases where data are costly to procure; another problem arises in higher dimensions, where shapes rapidly become more varied and complicated and classical algorithms for object identification quickly become difficult to produce. 
We combine machine learning algorithms with sparse data collection techniques to help overcome both problems.

The method we explore here is the {\it ray-based classification} (RBC) framework, which utilizes information about large $N$-dimensional data sets encoded in a collection of one-dimensional objects, called {\it rays}. 
Ultimately, we wish to explore the theoretical limits of how few data---how few rays, in our case---are required for resolving features of various sizes and levels of detail. 
In this paper, we determine these limits when the objects to be classified are convex polytopes.

The RBC framework measures convex polytopes by choosing a so-called observation point within the polytope, shooting a number of rays as evenly spaced as possible from this point, and recording the distance it takes for each ray to encounter a face.
While it is reasonable to expect that an explicit algorithm for recognizing polygons in a plane can be developed, in arbitrary dimension, such an explicit algorithm would be tedious to produce and theoretically unenlightening.
Since our work presumes a high cost of data acquisition but not computing power, we leave the actual classification to a machine learning algorithm. 
This paper produces theoretical bounds on how {\it few} data are required for a neural network to reliably classify shapes.

This project originated in the context of quantum information systems, specifically in the problem of calibrating the state of semiconductor quantum dots to work as qubits. 
The various device configurations create an irregular polytopal tiling of a configuration space, and the specific shape of a polytope conveys useful information about the corresponding device state. 
Our goal is to map out these shapes as cost-effectively as possible.
Here, the cost arises because polytope edges are detected through electron tunneling events which places hard physical limits on data acquisition rates. 
Apart from this original application, the techniques we developed should be valuable in any situation where object classification must be done despite constraints on data acquisition.

\section*{Related Work}\label{sec:context}
In the broad field of data classification in $N$ = 2, 3, 4, \textit{etc}.\ dimensions, there are many unique approaches, often tailored to the constraints of the problem at hand. 
For example, higher-dimensional data can be projected onto lower dimensions to employ standard deep learning techniques such as 3D ConvNets~\cite{Shi2015, Cao2017, Lyu2020}. 
Multiple low-dimensional views of higher-dimensional data can be collected to ease data collection and recognition~\cite{Zhao2017}. 
Models such as ShapeNets~\cite{Wu2015} directly work with 3D voxel data.
Data collected using depth sensors can be presented as RGB-D data~\cite{Ward2019, Socher2012, Cao2016} or point clouds~\cite{Rusu2011, Soltani2017} representing the topology of features present. 
Often, depth information is sparsely collected due to limitations of the depth sensors themselves. 
Within the field of representing 3D or higher-dimensional data as point clouds, data can be treated in various ways, such as simply $N$-dimensional coordinates in space~\cite{Qi2017}, patches~\cite{Tretschk2020}, meshed polygons~\cite{Lieberknecht2011}, or summed distances of the data to evenly spaced central points~\cite{Ng2020}. 
However, unlike most point-cloud-based classification frameworks, the proposed RBC
directly relies on ordered sets of points for predictions.

Critically, the RBC approach is suited for an environment in which data can be collected in any vector direction in $N$ dimensional space while even coarse data collection of the total space would be practically too expensive or unfeasible. 

Historically, it is well-known that the complexity of any classification problem intensifies in higher dimensions. 
This is the so-called {\it curse of dimensionality} \cite{bellman1966dynamic}, which has a negative impact on generalizing good performance of algorithms into higher dimensions. 
In general, with each feature and dimension, the minimum data requirement increases exponentially.
This can be seen in the present work: according to Theorem \ref{thm:bound_in_RN}, the data requirement increases like $\sqrt{N}e^{\alpha{}N}$.
At the same time, in many applications data acquisition is very expensive, resulting in datasets with a large number of features and a relatively small number of samples per feature (so-called {\it High Dimension Low Sample Size} datasets~\cite{hall2005geometric}).
To address these problems, a number of algorithms have been proposed to effectively select the most important features in the dataset (see, {\it e.g.}, \cite{Vapnik95, Freund96, Breiman00, Vapnik00, Hofmann08}).

Within the field of quantum dots, presented here as an application of the RBC framework, several strategies for classification and tuning using various machine learning techniques have been implemented. 
Using variational auto-encoders, standard device measurements have been optimized to reduce the total number of measurements required~\cite{Lennon19-EMM} and to automate fine tuning in higher ($N>2$) dimensions~\cite{Esbroeck2020}. 
Machine learning-based binary classifiers have been used to classify 2D stability diagrams as either good or bad for further experimental use~\cite{Darulova2020}. 
Several different CNNs have been implemented to classify 2D dot data using experimental~\cite{Durrer2020}, simulated \cite{Kalantre17-MLD}, or a combination of both data types~\cite{Darulova2020preprint}.
A machine learning algorithm has even been implemented to explore up to an eight-dimensional space to optimally tune towards a desired experimental state \cite{Moon2020}.
More recently, the ray-based measurement scheme has been implemented in conjunction with an active learning algorithm to estimate convex polytopes defining quantum dot states in 3D and 4D~\cite{Krause21-ECP,Chatterjee21-AEC}.

\section*{Problem Formulation}\label{sec:ProbForm}
We begin with a convex region $\mathcal{Q}\subseteq\mathbb{R}^N$ along with a point $x_o$, the observation point, in the interior of $\mathcal{Q}$.
Given a unit vector $v$, the ray based at $x_o$ in the direction $v$ is
\begin{eqnarray}
    \mathfrak{R}_{x_o,v}
    \;=\;\{x_o\,+\,tv\;|\;t\in[0,\infty)\}.
\end{eqnarray}
The set of directions $v$ at $x_o$ is naturally parameterized by the unit sphere $\mathbb{S}^{N-1}$.
$M$ many directions $v_1,\dots,v_M\in\mathbb{S}^{N-1}$ produces $M$ many rays $\{\mathfrak{R}_i\}_{i=1}^M$, $\mathfrak{R}_i=\mathfrak{R}_{x_o,v_i}$ based at $x_o$.
Because $\mathcal{Q}$ is convex, in the direction $v_i$ there will be a unique distance $t_i$ at which the boundary $\partial\mathcal{Q}$ is encountered.
Given a set of directions and an observation point, the corresponding collection of distances is called the {\it point fingerprint}.
\begin{definition}
    Given a convex region $\mathcal{Q}$, a point $x_o\in\mathcal{Q}$, and a set of directions $\{v_i\}_{i=1}^M\subset{}\mathbb{S}^{N-1}$, the corresponding {\bf point fingerprint} is the vector
    \begin{eqnarray}
    \mathcal{F}(\mathcal{Q},x_o,\{v_i\}_{i=1}^M)\;\equiv\;\mathcal{F}_{x_o}
    \;=\;\big(t_1,\dots,t_M\big)
    \end{eqnarray}
    where $t_i\in(0,\infty]$ is unique value with $x_o+t_iv_i\in\partial\mathcal{Q}$.
\end{definition}
In practice, there will be an upper bound on what values the $t_i$ may take, which we call $T$.
If the ray does not intersect $\partial\mathcal{Q}$ prior to distance $T$, one would record $t_i=\infty$, indicating the region's boundary is effectively infinitely far away in that direction.

\begin{figure}[b]
	\centering
    \includegraphics[width=\linewidth]{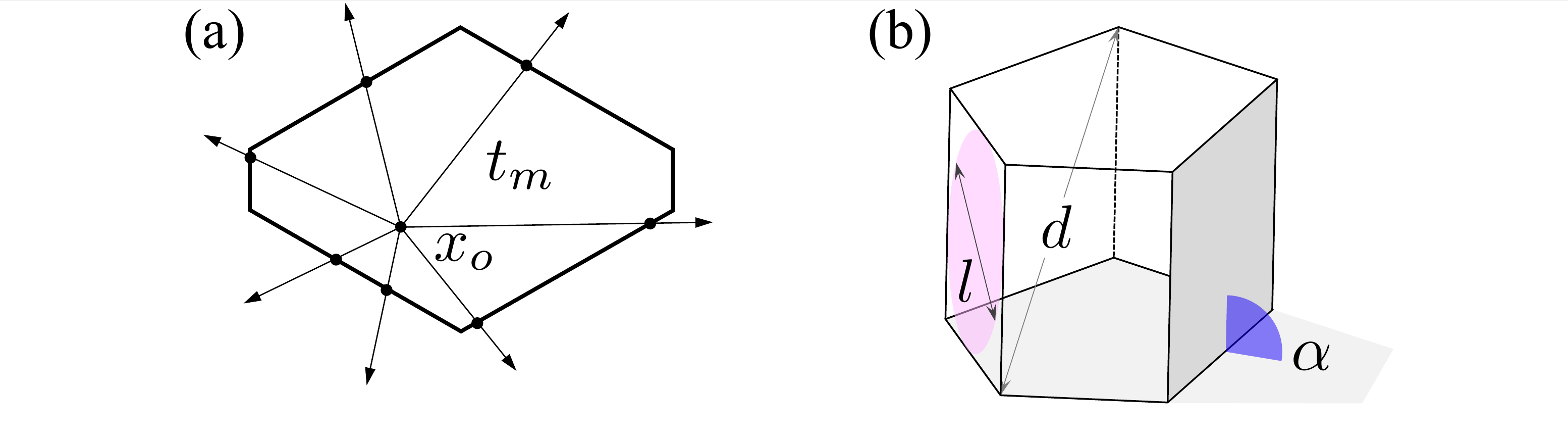}
	\caption{{\bf a} A sample polygon with 7 evenly spaced rays based at $x_o$, with $t_m$ denoting the distance from $x_o$ to the polygon edge $\partial\mathcal{Q}$. {\bf b} A depiction of a minimum interior diameter of a face $l$, the minimum exterior dihedral angle $\alpha$, and the maximum possible polytope diameter $d$ for a sample polytope in $\mathbb{R}^3$.}
	\label{fig:fig_1}
\end{figure}

The fingerprinting process is depicted in Fig.~\ref{fig:fig_1}a.
The question is to what extent one can characterize, or approximately characterize, convex shapes knowing only a fingerprint.
If nothing at all is known about the region $\mathcal{Q}$ except that it is convex, full recognition requires infinitely many rays measured in all possible directions, effectively resulting in measuring the entire $N$-dimensional space. 
However, it turns out that if one puts restrictions on what the objects could be---for instance, if it is known that $\mathcal{Q}$ must be a certain kind of polytope---information captured with a fingerprint may be sufficient.
Better yet, if we do not require a full reconstruction of the shape but only some coarser form of identification, for example, if we must distinguish triangles from hexagons but do not care exactly what the triangles or hexagons look like, then fingerprints can be made even smaller.

With an eye toward eventually approximating arbitrary regions with polytopes, we define the following polytope classes.
\begin{definition}
    Given $N\in\{2,3,\dots\}$ and $d,l,\alpha>0$, let $\mathcal{Q}(N,d,l,\alpha)$ be the class of convex polytopes in $\mathbb{R}^N$ that have diameter at most $d$, all face inscription sizes at least $l$, and all exterior dihedral angles at most $\alpha$.
\end{definition}

The ``inscription size'' of a polytope face is the diameter of the largest possible $(N-1)$-disk inscribed in that face. 
In the case $N=2$, polytopes are just polygons and polytope faces are line segments. 
In this case, the inscription size of a face is just its length. For the case of $N=3$, the inscription size of a face is the diameter of the largest possible disk inscribed in this face, see Fig.~\ref{fig:fig_1}b. We can now formulate the following identification problem.
\begin{problem}[The identification problem] \label{prob:identity}
    Given a polytope $\mathcal{Q}\in\mathcal{Q}(N,d,l,\alpha)$, determine the smallest $M$ so that, no matter where $x_o\in\mathcal{Q}$ is placed, a fingerprint made from no more than $M$ many rays is sufficient to completely characterize $\mathcal{Q}$.
\end{problem}
Again, the actual identification is done with a machine learning algorithm.
Resolving Problem \ref{prob:identity} will tell us how few data we can feed a neural network and still expect it to return a good identification. In $\mathbb{R}^2$, we actually solve this problem and find an optimal value of $M$.
In higher dimensions, we find a value for $M$ that works, but could be sharpened in some applications.

Hidden in Problem \ref{prob:identity} is another problem we call the {\it ray placement problem}. 
To explain this, note that a large number of rays may be placed at $x_o$, but if the rays are clustered in some poor fashion, very little information about the polytope overall geometry will be contained in the fingerprint.
This means that before one can determine {\it how many} rays are needed, one must already know {\it where} to place the rays.

In $\mathbb{R}^2$, this placement problem is easily solved: choosing a desired offset $v_0$, the $v_i$ are placed at intervals of $2\pi/M$ along the unit circle. 
In higher dimensions, the placement problem is much more difficult and we have to work with suboptimally spaced rays.
In fact, as we discuss later in this paper, even in $\mathbb{R}^3$ an optimal placement is out of reach.
To overcome this problem, we propose a general placement algorithm that works in arbitrary dimension and is reasonably sharp.
As we show, the proposed algorithm is sufficient to enable concrete estimates on the numbers of rays required to resolve elements in $\mathcal{Q}(N,d,l,\alpha)$.

In many practical applications, such as calibration of quantum dot devices mentioned earlier, Problem \ref{prob:identity} is much too strict.
Often we do not need to reconstruct polytopes exactly but only classify them to within approximate specifications.
For example, we may only wish to know if a triangle is ``approximately'' a right triangle, without needing enough data to fully reconstruct it.
Or we may wish to distinguish triangles and hexagons, and not care about other polyhedra. Theoretically, this involves separating the full polytope set $\mathcal{Q}(N,d,l,\alpha)$ into disjoint subclasses $\mathcal{C}_1,\dots,\mathcal{C}_K\subset\mathcal{Q}(N,d,l,\alpha)$, with possibly a ``leftover'' set $C_L=\mathcal{Q}(N,d,l,\alpha)\setminus\bigcup_{i=1}^K\mathcal{C}_i$ of unclassifiable or perhaps unimportant objects.
The idea is that an object's importance might not lie in its exact specifications, but in some characteristic it possesses.
\begin{problem}[The classification problem] \label{prob:classification}
    Assume $\mathcal{Q}(N,d,l,\alpha)$ has been partitioned into classes $\{\mathcal{C}_i\}_{i=1}^K$.
    Given a polytope $\mathcal{Q}$, identify the $\mathcal{C}_i$ for which $\mathcal{Q}\in\mathcal{C}_i$.
\end{problem}
The classification problem is more suitable for machine learning than the full identification problem.
This is in part because the outputs are more discrete (we can arrange it so the algorithm returns the integer $i$ when $\mathcal{Q}\in\mathcal{C}_i$), and in part because machine learning usually produces systems good at identifying whole classes of examples that share common features, while ignoring unimportant details.

Importantly, a satisfactory treatment of the classification problem can lead to solutions of more complicated problems, such as classifying compound items like tables, chairs, etc. in a 3D environment or geometrical objects obtained through measurements of an experimental variable in some parameter space.
Depending on the origin or purpose of such objects, they naturally belong to different categories.
For example, in the 3D real world, furniture and plants define two distinct classes that, if needed, can be further subdivided (e.g., a subclass of chairs, tables).
Objects belonging to a single class, in principle, share common characteristics or similar geometric features of some kind.

We close this section with two remarks.
The first is that the RBC framework has already seen considerable experimental success~\cite{Zwolak20-RBC}.
The second remark concerns a subordinate problem that is beyond the scope of this work: boundary identification.
In the quantum computing application for which RBC was originally designed \cite{Zwolak20-RBC} boundaries are identified by measuring discrete tunneling events, and there is little ambiguity in determining when a boundary was crossed.
Since the fingerprinting method relies on identifying boundary crossings, in other circumstances boundary detection might require some other resolution.
For now we only mention that machine learning methods should be able to compensate, to an extent, for boundaries that are indistinct or partially undetectable, and such algorithms often remain robust in the presence of noise.
We shall have more to say about this in future work.

\begin{figure}[t]
	\centering
    \includegraphics[width=\linewidth]{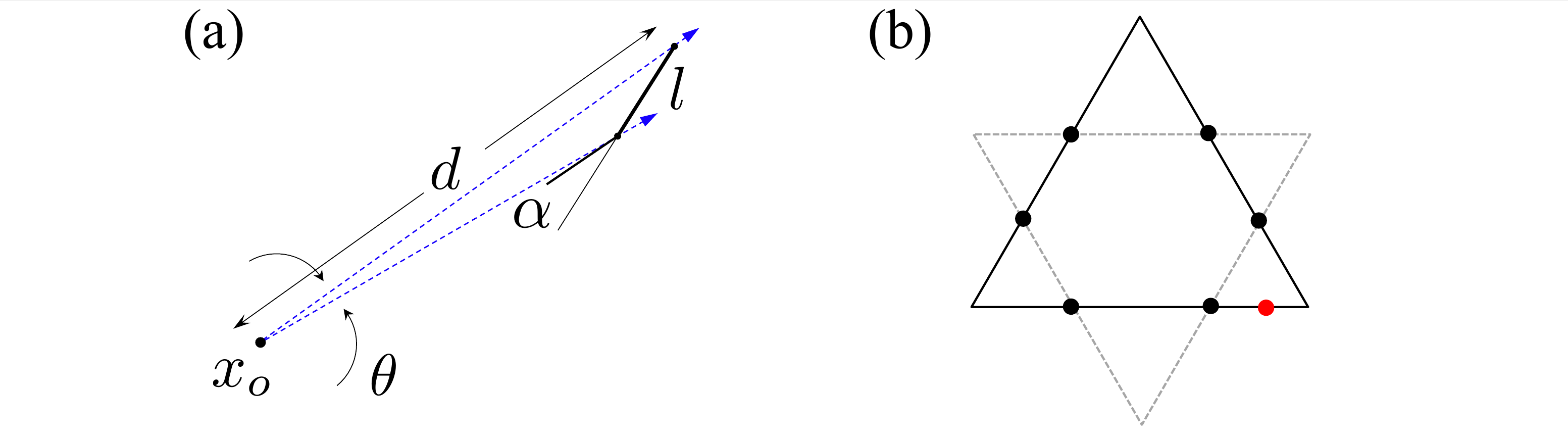}
	\caption{{\bf a} A depiction of the angular span, $\theta$ (marked with curved arrows). {\bf b} Ambiguity between a polygon $\mathcal{Q}$ (solid black) and its dual $\mathcal{Q}^*$ (dashed gray), resolved with a single additional intersection point marked in red.} 
	\label{fig:fig_2}
\end{figure}

\section*{Main Results} \label{sec:main_results}
A solution to Problem~\ref{prob:classification} in the supervised learning setting is obtained by training a deep neural network (DNN) with the input being the point fingerprint and an output identifying an appropriate class. Apriori it is unclear how many rays are necessary for a fingerprint-based procedure to reliably differentiate between polytopes. With data acquisition efficiency being the focus of this work, we want to theoretically determine the lower bound on the number of rays needed.
Such a bound is fully within reach for polygons in $\mathbb{R}^2$ (Theorem \ref{thm:bound_in_2D}), and can be approximated in all higher dimensions (Theorem \ref{thm:bound_in_RN}).

For a polytope face to be visible in a fingerprint, at least one ray must intersect it. To establish not only the presence of a face but its orientation in $N$-space, at least $N$ many rays must intersect it. The smaller a face is, the further away from the observation point $x_o$ it is, or the more highly skewed its orientation is, the more difficult it is for a ray to intersect it.
We address the case of polygons in $\mathbb{R}^2$ first, as we obtain the most complete information there.

\subsection*{\bf\emph{The Identification Problem in $\mathbb{R}^2$}}
\label{subsec:R2_pol_determ}
Recall that $\mathcal{Q}(2,d,l,\alpha)$ is the class of polygons in the plane with diameter $<d$, all edge lengths $>l$, and all exterior angles $<\alpha$.

\begin{theorem}[Polygon identification in $\mathbb{R}^2$]\label{thm:bound_in_2D}
    Assume $\mathcal{Q}$ is a polygon in $\mathcal{Q}(2,d,l,\alpha)$, and let $x_o$ be a point in the polygon's interior, from which $M$ many evenly spaced rays emanate. If
    \begin{equation}\label{eq:ray_number_bound}
    M\;>\;\Bigg\lceil\frac{4\pi}{\arcsin\left(\frac{l}{d}\sin\alpha\right)}\Bigg\rceil,  
    \end{equation}
    then two or more rays will intersect each boundary segment of $\mathcal{Q}$, and one segment will be hit at least 3 times.
\end{theorem}
The $\lceil\,\cdot\,\rceil$ notation indicates the usual ceiling function.
\begin{proof}
At the observation point $x_o$, each boundary segment has an {\it angular span}, defined to be the angle formed by joining $x_o$ to the segment's two endpoints; this is depicted by the angle $\theta$ in Fig.~\ref{fig:fig_2}a.
The idea is to compute the smallest possible angular span---which we call $\theta_{\min}$---given our constraints on $d$, $l$ and $\alpha$.
If we select $M$ such that $2\pi/M\le\frac12\theta_{\min}$, which is the same as selecting
$$M\;\ge\;\lceil4\pi/\theta_{\min}\rceil,$$
then the set of directions placed at intervals of $2\pi/M$ will intersect any angular interval of length $\ge\theta_{\min}$ a minimum of twice.
Consequently, the corresponding set of rays $\{\mathfrak{R}_i\}_{i=1}^M$ will intersect each boundary segment a minimum of twice.

From the Law of Sines, we find the smallest possible angular span to be $\theta_{\min}= \arcsin\left(\frac{l}{\,d\,}\sin\alpha\right)$, as depicted in Fig.~\ref{fig:fig_2}a.
We conclude that when
\begin{eqnarray}
    M\;\ge\;
    \left\lceil
    \frac{4\pi}{\theta_{\min}}
    \right\rceil
    \;=\;\left\lceil
    \frac{4\pi}{\arcsin\left(\frac{l}{d}\sin(\alpha)\right)}\right\rceil
    \label{eq:ray_number_bound_der}
\end{eqnarray}
and the directions are $v_i=v_0+2\pi{}i/M$, $i\in\{1,\dots,M\}$ (where $v_0$ is any desired offset), then the rays $\{\mathfrak{R}_i\}_{i=1}^M$ will intersect each polygon edge at least twice.

Replacing the ``$\ge$'' in (\ref{eq:ray_number_bound_der}) with ``$>$'' will ensure that each edge is hit by two rays, and at least one ray is hit by three rays. This concludes the proof.
\qed
\end{proof}

Knowing the location of two points on each edge is almost, but not quite, sufficient for identifying the polygon.
There remains an ambiguity between the polygon and its dual; see Fig.~\ref{fig:fig_2}b. This is resolved if at least one edge is hit 3 times. Thus Theorem \ref{thm:bound_in_2D} completely solves the identification problem in $\mathbb{R}^2$.

\begin{figure}[t]
	\centering
    \includegraphics[width=\linewidth]{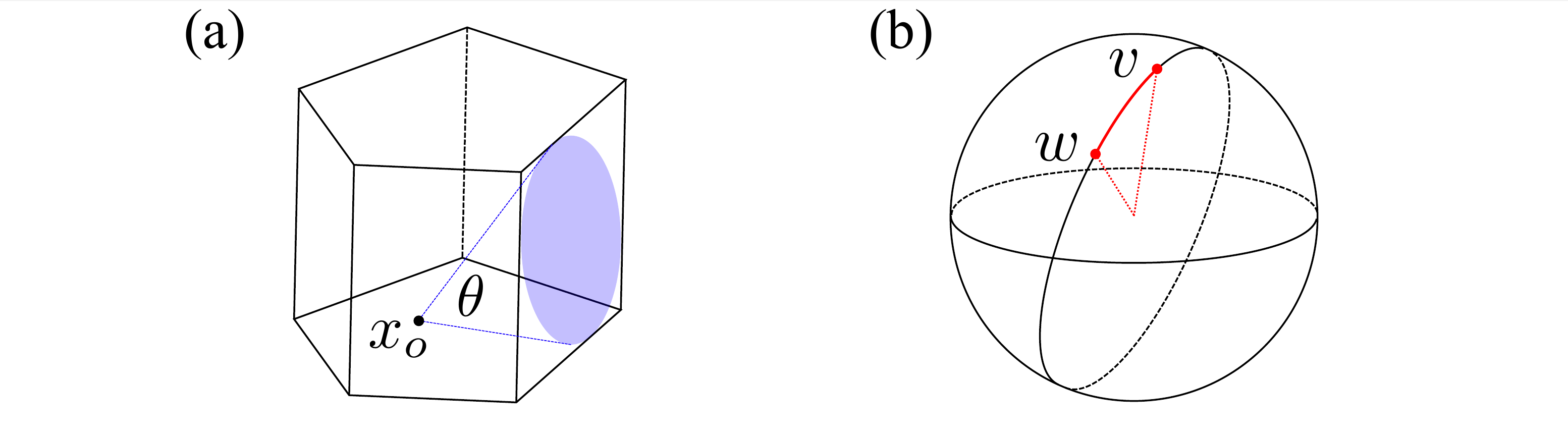}
	\caption{{\bf a} A depiction of the angular span of a face, $\theta$, for a sample polytope in $\mathbb{R}^3$. {\bf b} A visualization of the standard great-circle distance.} 
	\label{fig:fig_3}
\end{figure}

\subsection*{\bf\emph{The Identification Problem for Arbitrary Convex Polygons}}
\label{subsec:arb_pol_determ}
Identification in $\mathbb{R}^N$ follows a largely similar theory, with two substantial changes.
The first is that we must change what is meant by the angular span of a face, the second is that we must deal with the ray placement problem mentioned in Section ``Problem Formulation''.
The notion of angular span is relatively easily adjusted (see Fig.~\ref{fig:fig_3}a).
\begin{definition}[Angular span] \label{def:ang_span_Ndim}
    If $\mathcal{Q}$ is a convex polytope in $\mathbb{R}^N$, $N\ge2$, $x_o$ is an observation point in $\mathcal{Q}$, and $\mathcal{L}$ is a face of $\mathcal{Q}$, the {\it angular span} of $\mathcal{L}$ is the cone angle of the largest circular cone based at $x_o$ so that the cross-section of the cone that is created by plane containing $\mathcal{L}$ lies entirely within $\mathcal{L}$.
\end{definition}
We create a solution for the ray placement problem with an induction algorithm, but first we require some spherical geometry.
Given two points $v,w\in\mathbb{S}^{N-1}$, let $\Dist_{\mathbb{S}^{N-1}}(v,w)$ be the great-circle distance between them (see Fig.~\ref{fig:fig_3}b for visualization in $\mathbb{R}^3$).
Given $v\in\mathbb{S}^{N-1}$, we define a ball of radius $r$ on $\mathbb{S}^{N-1}$ to be
\begin{equation}
    \overline{B}_{v}(r)
    \;=\;\big\{
    w\in\mathbb{S}^{N-1}\;\big|\; \Dist{}_{\mathbb{S}^{N-1}}(v,w)\,\le\,r
    \big\}.
\end{equation}
For example, a ball $\overline{B}_v(\pi)$ of radius $\pi$ is the entire sphere itself, and any ball of the form $\overline{B}_{v}(\pi/2)$ is a hemisphere centered on $v$.
It will be important to know the $(N-1)$-area of the unit sphere $\mathbb{S}^{N-1}$, and also the $(N-1)$-area of any ball $\overline{B}_v(r)\subset\mathbb{S}^{N-1}$.
The standard area formulas from differential geometry are
\begin{equation}
    \begin{array}{rl}
    &\A\big(\mathbb{S}^{N-1}\big)
    =\frac{N\pi^{\frac{N}{2}}}{\Gamma(\frac{N}{2}+1)}, \\
    &\A\Big(\overline{B}_v(r)\Big)
    =\frac{(N-1)\pi^{\frac{N-1}{2}}}{\Gamma(\frac{N-1}{2}+1)}
    \int_0^r\sin{}^{N-2}(\rho)\,d\rho.
    \end{array}
    \label{eqn:AreaExact}
\end{equation}
The evaluation of $\int\sin^{N-2}(\rho)d\rho$ is a bit unwieldy,\footnote{A glance at the integral tables reveals $\int\sin^{N-2}(\rho)d\rho=-\cos(\rho)\,{}_2F_1\big(\frac12,\frac{3-N}{2};\frac32;\cos^2(\rho)\big)$ where ${}_2F_1$ is the usual hypergeometric function.}
but it will be enough to have the bounds
\begin{equation}
    \frac{\pi^{\frac{N-1}{2}}}{\Gamma(\frac{N+1}{2})}\sin^{N-1}(r)
    <
    \A\Big(\overline{B}_v(r)\Big)
    <
    \frac{\pi^{\frac{N-1}{2}}}{\Gamma(\frac{N+1}{2})}r^{N-1}. \label{eqn:AreaBounds}
\end{equation}

We also require the idea of the {\it density} of a set of points.
\begin{definition}[Density of points in $\mathbb{S}^{N-1}$]\label{def:density}
    Let $\mathcal{P}\subset\mathbb{S}^{N-1}$ be a finite collection of points $\mathcal{P}=\{v_1,\dots,v_k\}$, $v_i\in\mathbb{S}^{N-1}$ for $1\le i\le k$.
    We say that the set $\mathcal{P}$ is {\it $\bf\varphi$-dense} in $\mathbb{S}^{N-1}$ if, whenever $v\in\mathbb{S}^{N-1}$, then there is some $v_i\in\mathcal{P}$ with $\Dist_{\mathbb{S}^{N-1}}(v,v_i)\le\varphi$.
\end{definition}

We can now give a solution to the ray placement problem on $\mathbb{S}^{N-1}$.
We use an inductive point-picking process.
Pick a value $\varphi$; this will be the density one desires for the resulting set of directions on $\mathbb{S}^{N-1}$.
Begin the induction with any arbitrary point $v_1\in\mathbb{S}^{N-1}$.
If $\varphi$ is small enough that $\overline{B}_{v_1}(\varphi)$ is not the entire sphere, then we select a second point $v_2$ to be any arbitrary point not in $\overline{B}_{v_1}(\varphi)$.
Continuing, if points $v_1,\dots,v_{i}$ have been selected, let $v_{i+1}$ be any arbitrary point chosen under the single constraint that it is not in any $\overline{B}_{v_j}(\varphi)$, $j<i$.
That is, choose $v_{i+1}$ arbitrarily under the constraint
\begin{equation}
    v_{i+1}\in\mathbb{S}^{N-1}\setminus\left(\overline{B}_{v_1}(\varphi)\cup\dots\cup\overline{B}_{v_i}(\varphi)\right), \label{eqn:pt_pick_constraint}
\end{equation}
should such a point exist.
Should such a point {\it not} exist, meaning $\overline{B}_{v_1}(\varphi)\cup\dots\cup\overline{B}_{v_i}(\varphi)$ already covers $\mathbb{S}^{N-1}$, the process terminates, and we have our collection $\mathcal{P}=\{v_1,\dots,v_i\}$.

Whether an algorithm terminates or not is always a vital question.
This one does, and Lemma \ref{lem:placement} gives a numerical bound on its maximum number of steps.
This process requires numerous arbitrary choices---each point $v_i$ is chosen arbitrarily except for the single constraint that it not be in any of the $\overline{B}_{v_j}(\varphi)$, $j<i$---so it does not produce a unique or standard placement of points.
This contrasts to the very orderly choice of directions $v_i=v_0+2\pi{}i/M$ on $\mathbb{S}^1$ that we relied on in Theorem \ref{thm:bound_in_2D}.
Nevertheless, a set selected in this manner does have valuable properties, which we summarize in the following lemma.
\begin{lemma}[Properties of the placement algorithm] \label{lem:placement}
    Let $\mathcal{P}=\{v_1,v_2,\dots\}\subset\mathbb{S}^{N-1}$ be any set of points chosen using the inductive algorithm above.
    Then
    \begin{itemize}
        \item[(i)] the set $\mathcal{P}$ is $\varphi$-dense in $\mathbb{S}^{N-1}$, meaning that
        $\mathbb{S}^{N-1}=\bigcup_{v_i\in\mathcal{P}}\overline{B}_{v_i}(\varphi),$
        \item[(ii)] the half-radius balls $\overline{B}_{v_i}(\varphi/2)$ are mutually disjoint: $\overline{B}_{v_i}(\varphi/2)\cap \overline{B}_{v_j}(\varphi/2)=\varnothing$ when $i\ne{}j$, and
        \item[(iii)] the number of points in $\mathcal{P}$ is at most
        \begin{equation}
            M\;\le\;\sqrt{2\pi{}N}\left(\frac{1}{\sin(\varphi/2)}\right)^{N-1}. \label{eqn:BoundOnNumber}
        \end{equation}
    \end{itemize}
\end{lemma}
\begin{proof}
    We prove (ii) first.
    Without loss of generality suppose $i>j$.
    Recall the $i^{th}$ point $v_i\in\mathbb{S}^{N-1}$ was chosen under the single condition that $v_i\notin\bigcup_{j=1}^{i-1}\overline{B}_{v_j}(\varphi)$.
    This explicitly means $v_i$ is a distance greater than $\varphi$ from all the points that came before, so the balls of radius $\varphi/2$ around $v_i$ and $v_j$ cannot intersect.
    
    Next we prove (iii).
    Suppose there are $M$ many points in $\mathcal{P}$.
    Because the corresponding balls $\overline{B}_{v_i}(\varphi/2)$ are non-intersecting, we have the following:
    \begin{equation}
        \begin{array}{rl}
        \A\left(\mathbb{S}^{N-1}\right)
        &\ge
        \A\left(\bigcup_i\overline{B}_{v_i}(\varphi/2)\right) \\
        &=\sum_{i=1}^M\,\A\left(\overline{B}_{v_i}(\varphi/2)\right) \\
        &\;\ge\;M\cdot\frac{\pi^{\frac{N-1}{2}}}{\Gamma(\frac{N+1}{2})}\sin^{N-1}(\varphi/2).
        \end{array}
    \end{equation}
    Using (\ref{eqn:AreaExact}) this simplifies to
    \begin{equation}
        M\;\le\;\frac{\Gamma(\frac{N}{2}+\frac12)}{\Gamma(\frac{N}{2}+1)}N\sqrt{\pi}\frac{1}{\sin^{N-1}(\varphi/2)}.
    \end{equation}
    After noticing that $\frac{\Gamma(\frac{N}{2}+\frac12)}{\Gamma(\frac{N}{2}+1)}<\sqrt{2/N}$, we obtain (\ref{eqn:BoundOnNumber}).
    
    Lastly, we prove (i).
    We now know that the set $\mathcal{P}$ is a finite set, with a maximum number of elements given by (\ref{eqn:BoundOnNumber}).
    That means the inductive point-picking process used to create $\mathcal{P}$ must have terminated at some finite stage.
    If $\mathcal{P}=\{v_i\}$ was {\it not} $\varphi$-dense, there would be a point $v\in\mathbb{S}^{N-1}$ at distance greater than $\varphi$ from every $v_i$, that is $v\in\mathbb{S}^{N-1}\cap\bigcup_i\overline{B}_{v_i}(\varphi)$.
    However, because the point-picking process stopped exactly when there were {\it no more} such points to choose from, such a point $v$ cannot exist, and we conclude that $\mathcal{P}$ is $\varphi$-dense.
    \qed
\end{proof}

We can now proceed to the identification problem in $N$ dimensions.
\begin{theorem}[Polytope identification in $\mathbb{R}^N$]\label{thm:bound_in_RN}
    Assume $\mathcal{Q}\in\mathcal{Q}(N,d,l,\alpha)$.
    It is possible to choose a set of $M$ many directions $\{v_i\}_{i=1}^M$ so that given any observation point $x_o\in\mathcal{Q}$, the corresponding rays $\mathfrak{R}_i=\mathfrak{R}_{x_o,v_i}$ have the following properties:
    \begin{enumerate}
        \item The collection of rays $\{\mathfrak{R}_i\}_{i=1}^M$ strikes each polytope face $N$ or more times.
        \item The number of rays $M$ is no greater than 
        \begin{equation}\label{eq:bound_RN}
        M \;\le\;\sqrt{2\pi{}N}\left(\frac{1}{\sin(\frac{1}{12}\theta_{\min})}\right)^{N-1}
        \end{equation}
        where $\theta_{\min}=\arcsin\left(\frac{l}{\,d\,}\sin(\alpha)\right)$.
    \end{enumerate}
\end{theorem}
\begin{proof}
    We imitate the proof of Theorem \ref{thm:bound_in_2D}.
    Using again the Law of Sines, we compute the minimum angular span (see Definition \ref{def:ang_span_Ndim}) of any face of $\mathcal{Q}$ to be $\theta_{\min}=\arcsin\left(\frac{l}{\,d\,}\sin(\alpha)\right)$.
    
    Any circular cone with cone angle $\theta_{\min}$ creates a projection onto the unit sphere, and this projections is a ball of the form $\overline{B}_v(\frac{1}{2}\theta_{\min})$.
    We show that if $\mathcal{P}$ is a $\frac16\theta_{\min}$-dense set, then, inside {\it any} ball of radius $\frac12\theta_{\min}$ must lie at least $N$ many points of $\mathcal{P}$.
    
    The way we count the number points of $\mathcal{P}$ that must lie within $\overline{B}_v(\frac12\theta_{\min})$ is volumetrically.
    To give the idea, note that the balls $\{\overline{B}_{v_i}(\frac16\theta_{\min})\}_{i=1}^M$ cover all of $\mathbb{S}^{N-1}$ and so they must cover both $\overline{B}_v(\frac12\theta_{\min})$ as well as the sub-ball $\overline{B}_v(\frac13\theta_{\min})$.
    But for a ball $\overline{B}_{v_i}(\frac16\theta_{\min})$ to participate in the covering of $\overline{B}_v(\frac13\theta_{\min})$, it's center {\it must} lie within $\overline{B}_v(\frac12\theta_{\min})$.
    Using volumes to count up how many balls of radius $\frac{1}{6}\theta_{\min}$ it takes to cover a ball of radius $\frac{1}{3}\theta_{\min}$, we have an estimate of how many of the points $v_i$ lie in $B_v(\frac12\theta_{\min})$.
    See Fig. \ref{fig:fig_4}.
    
    \begin{figure}[t]
	    \centering
        \includegraphics[width=0.85\linewidth]{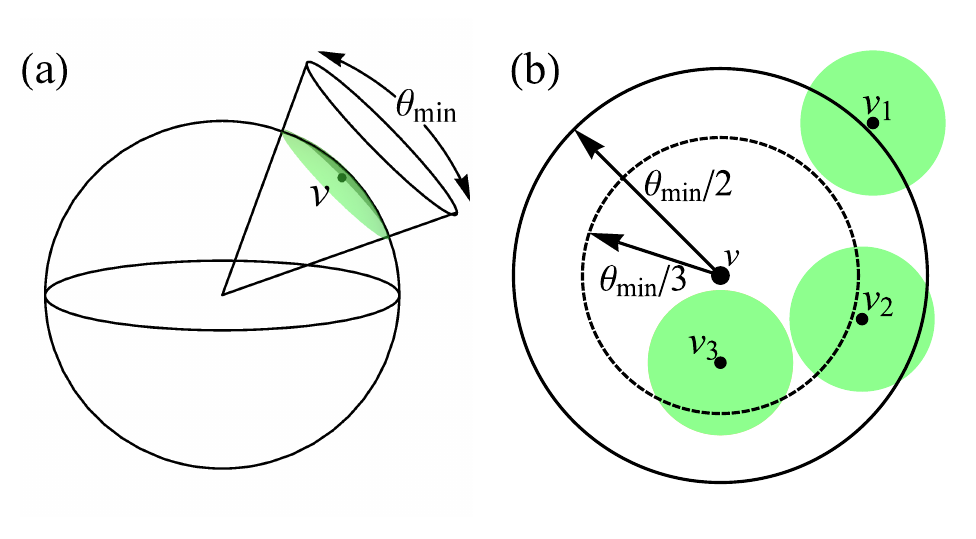}
	    \caption{{\bf a} Projection of a cone with cone angle $\theta_{\min}$ onto $\mathbb{S}^{N-1}$, creating the ball $\overline{B}_{v}(\frac12\theta_{\min})$. {\bf b} The covering argument: the centers $v_i\in\mathcal{P}$ of those balls of radius $\frac16\theta_{\min}$ which help cover $\overline{B}_v(\frac13\theta_{\min})$ must lie within $\overline{B}_v(\frac12\theta_{\min})$.} 
	    \label{fig:fig_4}
    \end{figure}
    
    Using (\ref{eqn:AreaBounds}), since $B_v(\frac13\theta_{\min})$ is covered with balls of radius $\frac16\theta_{\min}$, at least
    \begin{equation*}
        \begin{array}{rl}
        K\;=\;
        \frac{\A\left(\overline{B}_v(\frac13\theta_{\min})\right)}{\A\left(\overline{B}_v(\frac16\theta_{\min})\right)}
        &\;\ge\;
        \frac{\sin\left(\frac13\theta_{\min}\right)^{N-1}}{\left(\frac16\theta_{\min}\right)^{N-1}} \\
        &\;=\;
        \left(2\sinc(\scalebox{1}{$\frac{1}{3}$}\theta_{\min})\right)^{N-1}
        \end{array}
    \end{equation*}
    many balls of radius $\frac16\theta_{\min}$ participate in this cover.
    Thus, from above, at least $K$ many of the points of $\mathcal{P}$ lie within the slightly larger ball $\overline{B}_v(\frac12\theta_{\min}).$
    
    Since $\frac{l}{d}<1$, we can safely assume that $\theta_{\min}<\pi/2$, 
    that is $\arcsin(\frac{l}{d}\sin(\alpha))<\pi/2$.
    Therefore $2\sinc(\frac{1}{3}\theta_{\min})>2\sinc(\pi/6)\approx1.9$, and so $K\ge(1.9)^{N-1}$.
    We easily check that $(1.9)^{N-1}>N$ for $N\ge3$.
    We conclude that more than $N$ many balls of the form $B_{v_i}(\frac{1}{6}\theta_{\min})$ are part of the cover of $B_{v}(\frac{1}{3}\theta_{\min})$, and therefore greater than $N$ many of the points $v_i\in\mathcal{P}$ lie within $B_{v}(\frac{1}{2}\theta_{\min})$.
    
    To conclude, if $\mathcal{P}$ is the $\frac16\theta_{\min}$-dense set produced by the induction algorithm, we now know that (1) at least $N$ many corresponding rays must lie inside of any cone with cone angle $\theta_{\min}$ or greater by what we just proved, and (2) by Lemma \ref{lem:placement} it has fewer than $\sqrt{2\pi{}N}\csc^{N-1}(\frac{1}{12}\theta_{\min})$ elements.
    \qed
\end{proof}

The estimate (\ref{eq:bound_RN}) can be improved if our solution for the placement problem can be improved.
The {\it optimal} placement problem is unsolved in general; this and related problems go by several names, such as the hard spheres problem, the spherical codes problem, the Fejes T\'oth problem, or any of a variety of packing problems.
For a sampling of the extensive literature on this subject, see~\cite{Katanforoush03,Dumer07,Ballinger09,Saff97,Schutte1953,Conway2013}.
Our approach to this theorem, inspired by a technique of \cite{Gromov81}, was chosen because of its easy dimensional scalability---and as one moves through dimensions what is more important is the rate of increase with dimension rather than optimal coefficients.
Our result gives a theoretical bound in any dimension, and means of benchmarking and comparison.
In practice, for the modest number of rays and relatively low dimension, existing empirical algorithms are sufficient. 
In the case of larger numbers of rays or very high dimension, the placement algorithm prior to Lemma~\ref{lem:placement} could certainly be implemented.
The number of rays needed will usually grow exponentially in dimension, but given a fixed dimension the computational cost will be polynomial in the number of rays (this is similar to existing algorithms, although at present the details of how this compares to other algorithms is unknown).

\subsection*{\bf\emph{A classification problem example: The quantum dot dataset}}
To close the paper, we examine Problem~\ref{prob:classification} in the context of the quantum dot dataset studied by~\cite{Zwolak20-RBC}.
In this application, electrons are held within two potential wells of depths $d_1$ and $d_2$, which can be adjusted.
Depending on these values, electrons might be confined, might be able to tunnel between the two wells or travel freely between them, and might be able to tunnel out of the wells into the exterior electron reservoir.
Individual tunneling events can be measured, and, when plotted in the $d_1$-$d_2$ plane, create an irregular tiling of the plane by polygons.
The polygonal chambers represent discrete quantum configurations, and their boundaries represent tunneling thresholds.
The shape of a chamber provides information about the quantum state it represents.

The goal of~\cite{Zwolak20-RBC} was to map the $(d_1,d_2)$ configurations onto the quantum states of the device by taking advantage of the geometry of these polygons. With scalability being the overall objective, it was essential that the mapping requires as little input data as possible.
For theoretical reasons, it is known that each of the lattice's polygons belongs to one of six classes; roughly speaking, these are quadrilateral, hexagon, open cell (no boundaries at all), and three types of semi-open cells. 
Further, the hexagons themselves are known to be rather symmetric: they have center-point symmetry, with four longer edges typically of similar length, and two shorter edges of equal length (see Fig.~\ref{fig:fig_5}a).

\begin{figure}[t]
	\centering
    \includegraphics[width=0.9\linewidth]{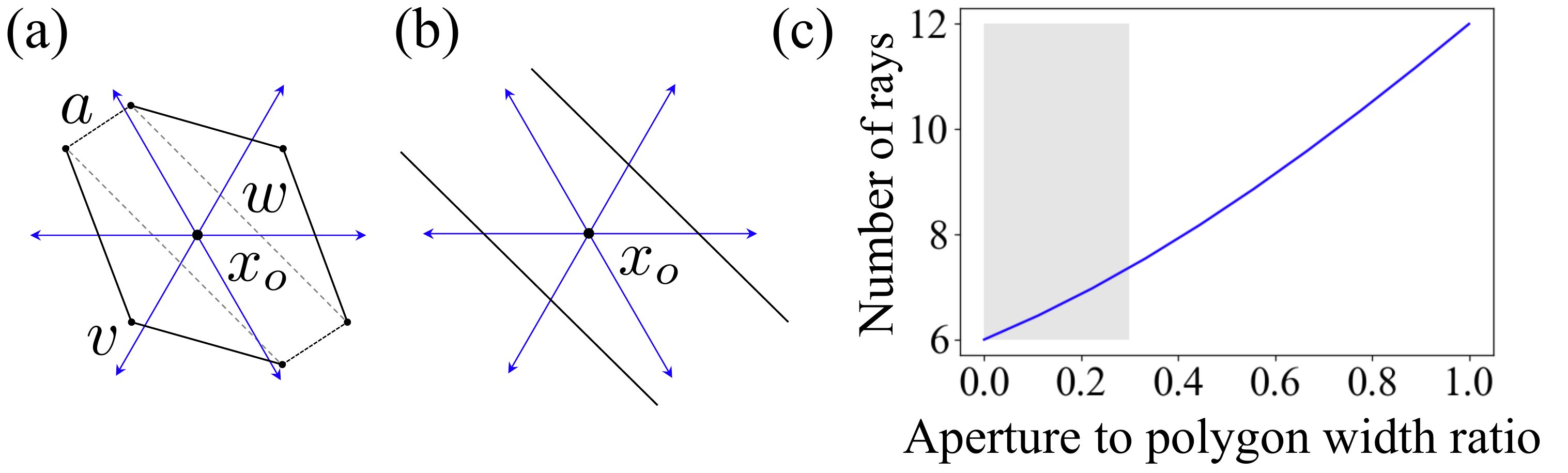}
	\caption{Schematics of two of the five geometrical shapes typical of the quantum dot dataset: {\bf a} a hexagon corresponding to a double-dot state and {\bf b} a strip contained by parallel lines corresponding to a singe-dot state.
	{\bf c} Plot of the lower bound $M$ on the number of rays to the ratio $\nicefrac{a}{w}$, as given by Eq.~(\ref{eq:QD-bound}). The shaded region corresponds to $\nicefrac{a}{w}$ ratios typical for real quantum dot devices.} 
	\label{fig:fig_5}
\end{figure}

In the language of Problem~\ref{prob:classification}, the interesting subclasses of polygons are $\mathcal{C}_1$: the hexagons with the symmetry attributes we described, including the quadrilaterals which are ``hexagons'' with $a=0$; $\mathcal{C}_2$, $\mathcal{C}_3$, $\mathcal{C}_4$: three kinds of semi-open cells contained between parallel or almost parallel lines; and $\mathcal{C}_5$: the open-cell, which has no boundaries at all.
The three classes of polygon $\mathcal{C}_2$, $\mathcal{C}_3$, $\mathcal{C}_4$ are distinguished from one another by their slopes in the $d_1$-$d_2$ plane: polygons in class $\mathcal{C}_2$ are between parallel lines with slopes between about 0 and $-\nicefrac12$, in class $\mathcal{C}_3$ between about $-\nicefrac12$, and about $-2$, and class $\mathcal{C}_4$ between about $-2$ and $-\infty$.
All other polygon types, for these purposes, are unimportant and can go in the ``leftover'' $\mathcal{C}_L$ category.
The question is how few rays are required to distinguish among the polygons within these classes.

In the quantum dot dataset, we must address one additional complication: the ``aperture,'' that is the shortest segment in Fig.~\ref{fig:fig_5}a, is sometimes undetectable.
The physical reason for this is that crossing this barrier represents electron travel between the two wells, and this event is often below the sensitivity of the detector. 

\begin{proposition}\label{thm:bound_qd}
    Let $x_o$ be an observation point which might be within a polygon of type $\mathcal{C}_1$--$\mathcal{C}_5$.
    Five rays are needed to distinguish these types.
    If the short segment is undetectable and the hexagon has the dimensions indicated in Fig.~\ref{fig:fig_5}a, then
    \begin{equation}\label{eq:QD-bound}
        M = \biggl\lceil\frac{6\pi}{\arccos\big(\frac{-1+(a/w)^2}{1+(a/w)^2}\big)}\biggl\rceil,
    \end{equation}
    many rays are needed to distinguish these types.
\end{proposition}
\begin{proof}
Referring to Fig.~\ref{fig:fig_5}a, the dimension $a$ is the dimension of the short side (the ``aperture''), and the dimension $w$ is the hexagon's width, specifically, the distance from an endpoint of one of its short segments to the corresponding endpoint on the opposite short segment, as represented by the two dotted segments in the hexagon of Fig.~\ref{fig:fig_5}a.

First consider a model situation of distinguishing between a line and two rays connected at a vertex.
To distinguish them, the arrangement must be hit with three or more sufficiently spaced rays: if the three rays' intersection points lie on a straight line then the object must be a line, whereas if they do not lie on a straight line we know the object must have a vertex.

Now consider a point $x_o$ placed within a hexagon, as shown in Fig.~\ref{fig:fig_5}a.
We require that either (1) three rays penetrate one of the the dotted lines of length $w$---so that a vertex can be detected as described in the previous paragraph---or (2) two rays penetrate one of the dotted lines, and 1 ray strikes either of the short segments of length $a$.

In the case that the segment $a$ is detectable, two of the longer line segments joined with a shorter segment will always occupy an angular width of at least $\pi$ from any observation point, no matter where it is placed.
For a minimum of three rays to find placement within any angular span of $\pi$, we require five rays.
Among five evenly spaced rays from a point between parallel lines, there are two possibilities: three will strike one line and two will strike the other, or two rays will strike each line and the fifth ray will be parallel to the other two and proceed to infinity---in either of these cases, parallel lines will be resolved along with their orientations in space.
This will also distinguish polygons that are closed (class $\mathcal{C}_1$, where no 3 rays will lie on any line) and polygons that are open (class $\mathcal{C}_5$, where the rays will hit nothing).

In the case that the segment $a$ is not detectable, either pair of two longer segments joined at a vertex must be struck three times.
From inside the polygon, the smallest possible angular span of either pair of two joined long segments is
\begin{equation}
    \theta_{\min}=\arccos\left(\frac{-1+(\nicefrac{a}{w})^2}{1+(\nicefrac{a}{w})^2}\right).  
\end{equation}
A minimum of three rays from $x_o$ must lie within this angular span.
Thus using $M$ rays evenly spaced about the full angular span $2\pi$ of the circle, we find the lower bound on the number of rays is
\begin{equation}
    M = \left\lceil\frac{2\pi}{\nicefrac{\theta_{\min}}{3}}\right\rceil,
\end{equation}
as claimed in Eq.~(\ref{eq:QD-bound}).
\qed
\end{proof}

To close the paper, we compare the theoretical bound given by Eq.~(\ref{eq:QD-bound}) with the performance of a neural network trained to recognize the difference between strips and hexagons.
The question is whether a neural network can come close to the theoretical ideal.

In fact it can.
In actual quantum dot environments, we expect values of $a$ to lie between about $0$ (where the hexagon degenerates to a quadrilateral) and about $\frac12w$; see, for example, Fig.~2 in~\cite{Zwolak20-RBC}.
For these values of $\nicefrac{a}{w}$, Eq.~(\ref{eq:QD-bound}) gives theoretical bounds on the necessary number of rays between six and about nine.
Empirical training experiments discussed in~\cite{Zwolak20-RBC} confirm that six rays and a relatively small DNN---that is a DNN with three hidden fully connected layers of 128, 64, and 32 neurons, respectively---are in fact sufficient to obtain classification accuracy of $96.4\,\%$ (averaged over $50$ training and testing runs, standard deviation $\sigma=0.4\,\%$).
This performance is on par with a ConvNet-based classifier using two-dimensional (2D) images of the shapes for which average accuracy of $95.9\,\%$ ($\sigma=0.6\,\%$) over 200 training and testing runs was reported~\cite{Zwolak18-QLD}. 
More recently, the RBC has been verified using experimental data, both off-line (i.e., by sampling rays from pre-measured large 2D scans) and on-line (i.e., by directly measuring the device response in a ray-based fashion)~\cite{Zwolak21-RBI}.
That paper found the RBC outperformed the more traditional 2D image-based classification of experimental quantum dot data that relied on convolutional neural network while requiring up to $70\,\%$ fewer data points. 
All tests reported in this section were performed on a 2019 MacBook Pro with 2.8 GHz Quad-Core Intel Core i7 processor.

\section*{Conclusions and Outlook}\label{sec:summary}
In conclusion, we have explored the ray-based classification framework for convex polytopes. 
We have proven a lower bound on the number of rays for shape identification in two dimensions and generalized the results to arbitrary higher dimensions. 
Finally, we discussed these results in context of the quantum dot dataset, which was the real-life application that motivated the RBC framework. 

Since objects in $N$-dimensional space can be approximated by convex polytopes, provided they are suitably rectifiable, this seemingly restricted technique opens the way to generalization. 
The problem of dividing a complicated object into a set of approximating polytopes can be considered a form of salience recognition and data compression---of detecting and storing the most useful or important features of the object. 
When the datum itself is scarce or costly to procure, one seeks methods that economize on input data while retaining salient features, even at the expense of some accuracy loss or potentially requiring heavier computing resources. 
RBC incorporating multiple intersections of the rays can be extended to solve problems where multiple nested shapes are present enclosing the observation point.

The approach of this paper gives good estimates on how few data are necessary to distinguish convex objects in arbitrary dimension, using the ray-based technique.
This is an important step as with the unavoidably high computational demands needed to study higher-dimensional datasets, one wishes for a way to tell ahead of time what the smallest possible resource demands might be.
Left for future work is installing and testing practical solutions in $N$ dimensions.
The problem of creating data classes in higher dimensions, which is necessary for Problem 2 to be well defined, is not difficult in dimensions 2 or 3, but present greater difficulties as the number of dimensions grows.
For example, it is unclear to what extent the RBC extends to cases where the number of possible convex polytopes is not know apriori.
Efficient division of geometric objects into usable classes is one way of reducing data requirements and complexity, but is unaddressed in the present study and represents a future avenue to explore.
Another interesting question, also not tackled in the current work, is the utility of the RBC framework to go beyond only assigning a class, to potentially reconstructing an enclosing convex polytope to some specified degree (a qualitative approximation of Problem 1).
In light of these open questions as well as the recently reported experimental validation of the utility of RBC for classifying states of quantum dot devices~\cite{Zwolak21-RBI}, the ray-based data acquisition combined with machine learning appears to be a very promising path forward.

\section*{Acknowledgements}
This research was sponsored in part by the Army Research Office (ARO), through Grant No. W911NF-17-1-0274. S.K. gratefully acknowledges support from the Joint Quantum Institute (JQI)--Joint Center for Quantum Information and Computer Science (QuICS) Lanczos graduate fellowship. The views and conclusions contained in this paper are those of the authors and should not be interpreted as representing the official policies, either expressed or implied, of the ARO, or the U.S. Government. The U.S. Government is authorized to reproduce and distribute reprints for Government purposes notwithstanding any copyright noted herein. Any mention of commercial products is for information only; it does not imply recommendation or endorsement by the National Institute of Standards and Technology.

\section*{Conflict of interest}
The authors declare they have no conflict of interest.


\end{document}